\documentclass[5p,twocolumn,preprint]{elsarticle}

\usepackage{subcaption}
\usepackage{amsthm,amsmath,amssymb, amsfonts}
\usepackage[linesnumbered,ruled]{algorithm2e}

\newtheorem{theorem}{Theorem}

\newcommand{\set}[1]{\left\{#1\right\}}
\newcommand{\abs}[1]{\left|#1\right|}
\newcommand{\norm}[1]{\left\|#1\right\|}
\newcommand{\ceil}[1]{\left\lceil#1\right\rceil}
\newcommand{\ip}[2]{\left\langle#1,#2\right\rangle}
\newcommand{\mat}[2]{\left[\begin{array}{#1} #2 \end{array}\right]}

\newcommand{\Scal}{\mathcal{S}}
\newcommand{\Rbb}{\mathbb{R}}
\newcommand{\Zbb}{\mathbb{Z}}

\begin{document}

\begin{frontmatter}


\title{Learning to Emulate an Expert Projective Cone Scheduler}

\author{Neal Master} 
\ead{neal.m.master@berkeley.edu} 
%
\address{{}}

\begin{abstract}

  Projective cone scheduling defines a large class of
  rate-stabilizing policies for queueing models relevant to several
  applications. While there exists considerable theory on the
  properties of projective cone schedulers, there is little practical
  guidance on choosing the parameters that define them. In this paper,
  we propose an algorithm for designing an automated projective cone
  scheduling system based on observations of an expert projective cone
  scheduler. We show that the estimated scheduling policy is able to
  emulate the expert in the sense that the average loss realized by
  the learned policy will converge to zero.  Specifically, for a
  system with $n$ queues observed over a time horizon $T$, the average
  loss for the algorithm is
  $O\left(\ln(T)\sqrt{\ln(n)/T}\right)$. This upper bound holds
  regardless of the statistical characteristics of the system. The
  algorithm uses the multiplicative weights update method and can be
  applied online so that additional observations of the expert
  scheduler can be used to improve an existing estimate of the
  policy. This provides a data-driven method for designing a
  scheduling policy based on observations of a human expert.  We
  demonstrate the efficacy of the algorithm with a simple numerical
  example and discuss several extensions.

\end{abstract}

\end{frontmatter}

\section{Introduction\label{sec:intro}}

In a variety of application areas, processing systems are dynamically
scheduled to maintain stability and to meet various other
objectives. Indeed, the basic problem in scheduling theory has been to
find and study policies that accomplish this task under different
modeling assumptions. In practice however, while human experts may be
able to manage real-world processing systems, it is typically
non-trivial to precisely quantify the costs and objectives that govern
expert schedulers. For example, in operating room scheduling,
\textit{ad~hoc} metrics have been applied in an attempt to model the
cost of delays, e.g.~\cite{Master_Prediction_2017}, but these metrics
are largely subjective. The Delphi Method is commonly used in
management science to quantitatively model expert opinions but such
methods have no algorithmic guarantees and are not always
reliable~\cite{Okoli_2004}.

In this paper, we present an online algorithm that allows us to
emulate an expert scheduler based on observations of the backlog of
the queues in the system and observations of the expert's scheduling
decisions. We use the term ``emulate'' to mean that while the
parametric form of the learned policy may not converge to the
parametric form of the expert policy in all cases, it will always
yield scheduling decisions that on average converge to the expert's
decisions.  This offers a data-driven way for designing autonomous
scheduling systems. We specifically consider a projective cone
scheduling (PCS) model which has applications in manufacturing,
call/service centers, and in communication
networks~\cite{Armony_Cone_2003,Ross_PCS_2009}.

The algorithm in this paper uses the multiplicative weight update
(MWU) method~\cite{Arora_MWU_2012}.  The MWU method has been used in
several areas including solving zero-sum games~\cite{Freund_1999},
solving linear programs~\cite{Plotkin_1995}, and inverse
optimization~\cite{Barmann_ICML_2017}. Because the PCS policy can be
written as a maximization, our techniques are most similar to those
used in~\cite{Barmann_ICML_2017}. In~\cite{Barmann_ICML_2017}, the
authors apply an MWU algorithm over a fixed horizon to learn the
objective of an expert who is solving a sequence of linear
programs. Our results differ from~\cite{Barmann_ICML_2017} in several
ways. One is that because of the queueing dynamics that we consider,
our expert's objective will vary over time whereas
in~\cite{Barmann_ICML_2017} the objective is constant. A related issue
is that in~\cite{Barmann_ICML_2017}, when the expert has a decision
variable of dimension $n$, the dimension of the parameter being
learned is also $n$. In our case, when the expert has a decision
variable of dimension $n$ (i.e. there are $n$ queues in the system),
we need to estimate $\Theta(n^2)$ parameters.  We also note that in
this paper we provide an algorithm that can be applied even when the
horizon is not known \textit{a~priori}.

The goal of inferring parts of an optimization model from data is a
well-studied problem in many other applications. For example, genetic
algorithm heuristics have been applied to estimate the objective and
constraints of a linear program in a data envelopment analysis
context~\cite{Troutt_2008}. The goal of imputing the objective
function of a convex optimization problem has also been considered in
the optimization community,
e.g.~\cite{Keshavarz_2011,Thai_2016}. These papers rely heavily on the
convexity of the objective and the feasible set. This approach does
not apply in a PCS context because the set of feasible scheduling
actions is discrete and hence non-convex.

This paper is also related to inverse reinforcement learning. Inverse
reinforcement learning is the problem of estimating the rewards of a
Markov decision process (MDP) given observations of how the MDP
evolves under an optimal policy~\cite{Ng_IRL_2000}. Inverse
reinforcement learning can be used to emulate expert decision makers
(referred to as ``apprenticeship learning'' in the machine learning
community) as long as the underlying dynamics are
Markovian~\cite{Abbeel_Apprenticeship_2004}. In the PCS model, no such
assumption is made and so our results naturally do not require
Markovian dynamics.

The remainder of this paper is organized as
follows. Section~\ref{sec:dynamics} specifies the PCS model that we
consider. Section~\ref{sec:algos} presents our algorithms and the
relevant guarantees. Because we take a MWU approach to the problem,
our guarantees are bounds on the average loss. However, we also
provide a concentration bound which gives guarantees on the tail of
the loss distribution. We provide a simple numerical demonstration of
our algorithms in Section~\ref{sec:sim}. In
Section~\ref{sec:extensions} we discuss some extensions of our results
and we conclude in Section~\ref{sec:conclusions}.

\section{Projective Cone Scheduling Dynamics\label{sec:dynamics}}

In this section we summarize the PCS model presented in
\cite{Ross_PCS_2009} and comment on the connection to the model
presented in \cite{Armony_Cone_2003}. The PCS model has $n$ queues
each with infinite waiting room following an arbitrary queueing
discipline. Time is discretized into slots $t \in \Zbb_+$\footnote{We
  use the notation $\Zbb_+ = \set{0, 1, 2, \hdots}$.}. The backlog in
queue $i$ at the beginning of time slot $t$ is $x_t(i)$. The backlog
across all queues can be written as a vector $x_t \in \Zbb_+^n$.  The
number of customers that arrive at queue $i$ at the end of time slot
$t$ is $a_t(i)$. The arrivals across all queues can be written as a
vector $a_t \in \Zbb_+^n$. Scheduling configurations are chosen from a
finite set $\Scal \subsetneq \Zbb_+^n$. If configuration $s_t \in
\Scal$ is chosen in time slot $t$ then for each queue $i$,
$\min\set{s_t(i), x_t(i)}$ customers are served at the beginning of
the time slot. We take the departure vector as $d_t = \min\set{s_t,
  x_t} \in \Zbb_+^n$ where the minimum is taken component-wise. This
gives us the following dynamics
\begin{equation}
  x_{t + 1} = x_t - d_t + a_t
  \label{eq:dynamics}
\end{equation}
where $x_0 \in \Zbb_+^n$ is arbitrary. Note that the arrival vector is
allowed to depend on previous scheduling configurations, previous
arrivals, and previous backlogs in an arbitrary way. 

The scheduling configurations are dynamically chosen by solving the
maximization
\begin{equation}
  \max_{s \in \Scal} \ip{s}{B x_t} = \max_{s \in \Scal} \sum_{i, j} s(i)B(i, j) x_t(j)
\end{equation}
where $B \in \Rbb^{n \times n}$ is symmetric and positive-definite
with non-positive off-diagonal elements. We assume that $\Scal$ is
endowed with some arbitrary ordering used for breaking ties. This PCS
policy defines a broad class of scheduling policies and in particular
we note that by taking $B$ as the identity matrix, we return to the
typical maximum weight matching scheduling algorithm.

Although $B$ is a matrix, because $B$ is symmetric, there are only $p
= n(n+1)/2$ rather than $n^2$ free parameters that need to be
learned. Consequently, we will represent the projective cone scheduler
with an upper-triangular array rather than a matrix. In particular,
take $b(i, i) \propto B(i, i)$ for $i \in [n]$ and $b(i, j) \propto
-B(i, j)$ for $(i, j) \in \set{(i, j) \in [n] \times [n] : i <
  j}$\footnote{We use the notation $[k] = \set{1, 2, \hdots, k}$.}. We
can also assume without loss of generality that $\sum_{i, j} b(i, j) =
1$.  Then we can write the projective cone scheduling decision as
follows:
\begin{align}
  s_t &= \arg\max_{s \in \Scal} \ip{s}{Bx_t}\nonumber\\
  &= \arg\max_{s \in \Scal} \sum_{i=1}^n \sum_{j=1}^n B(i,j) s(i) x_t(j)\nonumber\\
  &= \arg\max_{s \in \Scal} \Big\{\sum_i B(i, i) s(i) x_t(i) \nonumber\\
  &\hspace{1.75cm}+ \sum_{i < j} B(i,j) (s(i) x_t(j) + s(j) x_t(i))\Big\}\nonumber\\
  &= \arg\max_{s \in \Scal} \Big\{\sum_i b(i, i) s(i) x_t(i) \nonumber\\
  &\hspace{1.75cm}- \sum_{i < j} b(i,j) (s(i) x_t(j) + s(j) x_t(i))\Big\}\nonumber\\
  &\triangleq \mu(x_t; b)
\end{align}

For convenience, let us define 
\begin{equation}
  \sigma(i, j) = \left\{
    \begin{array}{rc}
      1 &,\, i = j\\
      -1 &,\, i \neq j
    \end{array}
  \right.
\end{equation}
so that we we can write $\mu$ more compactly as follows:
\begin{align}
  \mu(x_t; b) %
  &= \arg\max_{s \in \Scal} \sum_{i \leq j} \sigma(i, j)b(i, j) (s(i) x_t(j) + s(j) x_t(i))
\end{align}

Note that if we define $y_t$, a normalized version of $x_t$, as follows,
\begin{equation}
  y_t = \left\{
    \begin{array}{cl}
      x_t &,\, \norm{x_t}_1 = 0\\
      x_t/\norm{x_t}_1 &,\, \text{otherwise}
    \end{array}
  \right.
\end{equation}
then we have that $s_t = \mu(x_t; b) = \mu(y_t; b)$.

Modeling each customer as having a uniform deterministic service time
is motivated largely by applications in computer systems and in
particular, packet switch scheduling. However, PCS models with
non-deterministic service times have also been considered in the
literature \cite{Armony_Cone_2003}. However, the results in
\cite{Armony_Cone_2003} only apply to the case where $B$ is
diagonal. We have opted to present our algorithms in the context of
the non-diagonal case because we feel that having $\Theta(n^2)$
parameters is more interesting than having only $\Theta(n)$
parameters. Our algorithms can still be applied in the case of
stochastic service times; this is discussed along with other
extensions in Section~\ref{sec:extensions}.

Finally, we note that previous literature on PCS
\cite{Armony_Cone_2003,Ross_PCS_2009} has required a variety of
additional assumptions. For example, in \cite{Ross_PCS_2009} it is
assumed that the arrival process is mean ergodic. We do not require
such an assumption and moreover, while the results in
\cite{Armony_Cone_2003} and \cite{Ross_PCS_2009} are primarily
stability guarantees, we make no assumptions on the stability of the
system.

\section{The Learning Algorithm\label{sec:algos}}

In this section we present our algorithm. We first present a finite
horizon algorithm and then leverage this to present an infinite
horizon algorithm. For both algorithms, we show that the average error
is $O(\ln(T)\sqrt{\ln(p)/T})$. We also provide a bound on the fraction
of observations for which the error exceeds our average case bound.

These algorithms are applied causally in an online fashion. Although
we do not focus on computational issues, we note that computing
$\mu(y_t; b)$ is generally a difficult problem. However, there are
local search heuristics that allow efficient computation of
$\mu(y_{t+1}; b)$ based on the solution to $\mu(y_t; b)$
\cite{Ross_LocalSearch_2004}. Our algorithms require the computation
of $\mu(y_t; \hat b_t)$ where $\hat b_t$ is the estimate of $b$ at
time $t$ and so an online algorithm is appropriate if we want to use
the previous solution as a warm start.

Before presenting the algorithms, we consider the loss function of
interest. Since the expert we are trying to emulate is specified by
the array $b$, it may seem reasonable to want our estimates $\hat b_t$
to converge to $b$. However, this goal is not as reasonable as it may
seem. Because $\Scal$ is discrete, it is possible that two different
values of $b$ can render the same scheduling decisions. Consequently,
the goal of exactly recovering $b$ may be ill-posed. We aim to emulate
the expert scheduler so we want $\hat s_t = \mu(y_t; \hat b_t)$, the
scheduling decision induced by the estimate $\hat b_t$, to be the same
as $s_t$, the expert's scheduling decision. Hence, the loss should
directly penalize discrepancies between $s_t$ and $\hat s_t$.  This
leads us to jointly consider $\hat b_t$ and $\hat s_t$ so that the
loss at time $t$ is
\begin{equation}
  \ell_t %
  = \sum_{i \leq j} \sigma(i, j) (\hat b_t(i, j) - b(i, j))%
  (\delta_t(i)y_t(j) + \delta_t(j) y_t(i))
\end{equation}
where $\delta_t = \hat s_t - s_t$. When $b = \hat b_t$, we have that
$s_t = \hat s_t$ and $\ell_t = 0$. In addition, when $s_t = \hat s_t$
we have that $\ell_t = 0$ even if $b \neq \hat b_t$. The definition of
$\mu$ will allow us to show below that $\ell_t \geq 0$.

Another advantage to this loss function is that it allows us to give
guarantees that are independent of the statistics of the arrival
process. For example, suppose that there are no arrivals at some
subset of the queues. In this case, it would be unreasonable to expect
to be able to estimate the rows and columns of $b$ relevant to those
queues. More generally, the arrival process may not sufficiently
excite all modes of the system. By considering $\hat s_t$ and $\hat
b_t$ simultaneously, we can provide bounds that apply even in the
presence of pathological arrival processes.

\subsection{A Finite Horizon Algorithm}

We first present Algorithm~\ref{algo:finite}, a finite horizon
algorithm that requires knowledge of the
horizon. Algorithm~\ref{algo:finite} is a multiplicate weights update
algorithm and this time horizon is used to set the learning rate.

\begin{algorithm}
  \SetKwInOut{Input}{Input}
  \SetKwInOut{Output}{Output}
  
  \Input{$((y_1, s_1), \hdots, (y_T, s_T))$ \# observations}

  \Output{$(\hat b_1, \hdots, \hat b_T)$ \# parameter estimates}

  \Output{$(\hat s_1, \hdots, \hat s_T)$ \# scheduling estimates}

  $\eta \leftarrow \sqrt{\ln(p)/T}$ 

  $w_1 \leftarrow \text{upper triangular array of 1s}$

  \For{$t = 1, \hdots T$}{

    $\hat b_t \leftarrow w_t / \sum_{i \leq j} w_t(i, j)$

    $\hat s_t \leftarrow \mu(y_t; \hat b_t)$

    $m_t \leftarrow \text{upper triangular array of 0s}$

    $\delta_t = \hat s_t - s_t$

    \If{$\hat s_t \ne s_t$}{

      $z_t = \delta_t / \norm{\delta_t}_\infty$

      \For{$(i, j) \in [n]^2 : i \leq j$}{

        $m_t(i, j) \leftarrow \sigma(i, j)(z_t(i) y_t(j) + z_t(j)y_t(i))$

      }

    }

    $w_{t+1} \leftarrow w_t(1 - \eta m_t)$ \# component-wise

  }

  \caption{\newline\label{algo:finite} Online Parameter Learning with a Fixed Horizon}
\end{algorithm}

\begin{theorem}
  \label{thrm:finite}
  Let 
  \begin{equation*}
    D = \max_{(u, v) \in \Scal^2} \norm{u - v}_\infty
  \end{equation*}
  and $p = n(n+1)/2$.  If $T > 4\ln(p)$ then the output of
  Algorithm~\ref{algo:finite} satifies the following inequality:
  \begin{align}
    0 \leq \frac{1}{T} \sum_{t=1}^T \ell_t \leq 2D\sqrt{\frac{\ln(p)}{T}}
  \end{align}

\end{theorem}
\begin{proof}
  Note that because $m_t \in [-1, 1]^p$ and $\eta < \frac{1}{2}$ we
  can directly apply \cite[Corollary~2.2.]{Arora_MWU_2012}:
  \begin{align}
    \sum_{t = 1}^T &\sum_{i \leq j} m_t(i, j) \hat b_t(i, j) \nonumber\\
    &\leq \sum_{t = 1}^T \sum_{i, j}(m_t(i, j) + \eta \abs{m_t(i, j)})b(i, j) %
    + \frac{\ln(p)}{\eta}
  \end{align}
  Since $\abs{m_t(i, j)} \leq 1$ and $\sum_{i \leq j} \hat b_t(i, j) = 1$ we
  have the following:
  \begin{align}
    \sum_{t = 1}^T &\sum_{i \leq j} m_t(i, j) \hat b_t(i, j)  \nonumber\\
    &\leq \sum_{t = 1}^T \sum_{i, j}m_t(i, j)b(i, j) + \eta T + \frac{\ln(p)}{\eta}
  \end{align}
  A straightforward calculation shows that this upper bound is
  minimized when $\eta = \sqrt{\ln(p) / T}$. Rearranging the
  inequality and applying this fact give us the following:
  \begin{align}
    \frac{1}{T}\sum_{t = 1}^T &\sum_{i \leq j} m_t(i, j) \hat b_t(i, j) %
    - \frac{1}{T}\sum_{t = 1}^T \sum_{i \leq j} m_t(i, j) b(i, j)  \nonumber \\
    &\leq 2\sqrt{\frac{\ln(p)}{T}}
  \end{align}
  Now we apply the specifics of $m_t$. By definition of $D$,
  $\norm{\delta_t}_\infty \leq D$. This gives us the following:
  \begin{align}
    &\frac{1}{T}\sum_{t = 1}^T \sum_{i \leq j} \sigma(i, j) %
    (\delta_t(i) y_t(j) + \delta_t(j) y_t(i))%
    \hat b_t(i, j) \nonumber\\
    + &\frac{1}{T}\sum_{t = 1}^T \sum_{i \leq j} \sigma(i, j) %
    (-\delta_t(i) y_t(j) - \delta_t(j) y_t(i))%
    b(i, j)  \nonumber \\
    &\leq 2D\sqrt{\frac{\ln(p)}{T}}
    \label{ineq:cesaro}
  \end{align}
  Note that $\hat s_t = \mu(y_t; \hat b_t)$ and $\mu$ is defined in terms
  of a maximization. Therefore, 
  \begin{align*}
    \sum_{i \leq j} &\sigma(i, j) %
    (\hat s_t(i) y_t(j) + \hat s_t(j) y_t(i)) \hat b_t(i, j) \\
    &\geq \sum_{i \leq j} \sigma(i, j) %
    (s(i) y_t(j) + s(j) y_t(i)) \hat b_t(i, j)
  \end{align*}
  for any $s \in \Scal$. This shows that each term in the first
  Ces\`aro sum in~(\ref{ineq:cesaro}) is non-negative. Similarly, each
  term in the second Ces\`aro sum in~(\ref{ineq:cesaro}) is
  non-negative. This gives us a lower bound of zero. Rearranging the
  terms leaves us with the desired results.
\end{proof}

\subsection{An Infinite Horizon Algorithm}

We now present Algorithm~\ref{algo:infinite}, an infinite horizon
algorithm that dynamically changes the learning
rate. Algorithm~\ref{algo:infinite} applies the ``doubling trick'' to
Algorithm~\ref{algo:finite}. The idea is that we define epochs $[T_k,
T_{k+1}]$ where $T_k~=~2^k(4\ln(p))$ for $k \geq 0$ with $T_{-1} =
0$. The duration of the $k^{th}$ epoch is $T_k$ and in this epoch we
apply Algorithm~\ref{algo:finite}. Up to poly-logarithmic factors of
$T$, this gives us the same convergence rate that we had for
Algorithm~\ref{algo:finite}.

\begin{algorithm}
  \SetKwInOut{Input}{Input}
  \SetKwInOut{Output}{Output}
  
  \Input{$((y_1, s_1), (y_2, s_2), \hdots)$ \# observations}

  \Output{$(\hat b_1, \hat b_2, \hdots)$ \# parameter estimates}

  \Output{$(\hat s_1, \hat s_2, \hdots)$ \# scheduling estimates}

  $T_{-1} \triangleq 0$

  $T_k \triangleq 2^k (4\ln(p))$ for $k \in \set{0, 1, 2, \hdots}$.

  \For{$t = 1, 2, \hdots $}{

    \If{ $T_k < t \leq T_{k +1}$} {

      Apply Algorithm~\ref{algo:finite} 

      with $T\equiv T_k$ and without re-initializing $w_t$

    }

  }

  \caption{\newline\label{algo:infinite} Online Parameter Learning with an Unknown Horizon}
\end{algorithm}

\begin{theorem}
  \label{thrm:infinite}
  Suppose $T \geq T_0$. Define $\lg(\cdot)$ as
  $\ceil{\log_2(\cdot)}$. Then the output of
  Algorithm~\ref{algo:infinite} satifies the following inequality:
  \begin{align}
    0\leq %
    \frac{1}{T}\sum_{t=1}^T \ell_t %
    \leq 2\sqrt{2}D \lg\left(\frac{2T}{T_0}\right) \sqrt{\frac{\ln(p)}{T}}%
    \label{ineq:infinite}
  \end{align}

  Note that these are the same bounds as in Theorem~\ref{thrm:finite}
  but with an additional factor of $\sqrt{2}\lg(2T / T_0)$. 
\end{theorem}
\begin{proof}
  First note that the proof of \cite[Corollary~2.2.]{Arora_MWU_2012}
  does not require the initial weights to be uniform so
  Theorem~\ref{thrm:finite} still applies even without the
  initialization on line 2 of Algorithm~\ref{algo:finite}.  For
  convenience, let $U_k = 2D\sqrt{\ln(p) / T_k}$ and take $K =
  \lg(T/T_0)$.  Applying Theorem~\ref{thrm:finite} to each stage of
  Algorithm~\ref{algo:infinite} gives us the following:
  \begin{align}
    0 &\leq \sum_{t=1}^T \ell_t \nonumber\\
    &\leq \sum_{k=0}^K \sum_{t=1}^{T_k} \ell_t %
    \leq \sum_{k=0}^K T_k U_k %
    \leq \sum_{k=0}^K 2D\sqrt{\ln(p)}\sqrt{T_k}\nonumber\\
    &\leq 2D(K + 1)\sqrt{\ln(p)}\sqrt{T_K} \nonumber\\
    &\leq 2\sqrt{2}D(K + 1)\sqrt{\ln(p)}\sqrt{T} 
  \end{align}
  The first inequality follows from the fact that $\ell_t \geq 0$; the
  second inequality follows by extending the sum from $T$ up to $T_K$;
  the third and fourth inequalities follow from
  Theorem~\ref{thrm:finite}. The penultimate inequality follows from
  the fact that $\set{T_k}$ is an increasing sequence and the final
  inequality follows because $T_K$ can be no more that $2T$.

  Dividing by $T$ gives the desired result.
\end{proof}

\subsection{A Concentration Bound}

Our previous results provided bounds on the average loss of our
algorithms. In this section, we provide bounds for the tail of the
distribution of the loss. This gives us the guarantee that the
fraction of observations for which the loss exceeds our average case
bound tends to zero. 

\begin{theorem}
  \label{thrm:concentration}
  Let
  \begin{equation}
    f_T(\epsilon) = \frac{\abs{\set{ t \leq T : \ell_t > 2\sqrt{2}D\lg\left(\frac{2T}{T_0}\right)\sqrt{\frac{\ln(p)}{T}} + \epsilon}} }{T}
  \end{equation}
  be the fraction of observations up to time $T \geq T_0$ for which the loss
  exceeds the average-case bound by at least $\epsilon$. Then for any
  $\epsilon > 0$ we have that
  \begin{equation}
    f_T(\epsilon) \leq 1 - \frac{\epsilon}{2\sqrt{2}D\lg\left(\frac{2T}{T_0}\right)\sqrt{\frac{\ln(p)}{T}} + \epsilon}
  \end{equation}
  and hence,
  \begin{equation}
    \lim_{T \rightarrow \infty} f_T(\epsilon) = 0.
  \end{equation}
\end{theorem}
\begin{proof}
  The observed loss sequence $\{\ell_t\}_{t = 1}^T$ defines a point
  measure on $\Rbb_+$ where each point has mass $1/T$. Applying
  Markov's Inequality to this measure gives us that
  \begin{align*}
    f_T(\epsilon) %
    &\leq \frac{1}{2\sqrt{2}D\lg\left(\frac{2T}{T_0}\right)\sqrt{\frac{\ln(p)}{T}} + \epsilon} \cdot %
    \frac{1}{T}\sum_{t=1}^T \ell_t\\
    &\leq \frac{1}{2\sqrt{2}D\lg\left(\frac{2T}{T_0}\right)\sqrt{\frac{\ln(p)}{T}} + \epsilon} \cdot %
    2\sqrt{2}D\lg\left(\frac{2T}{T_0}\right)\sqrt{\frac{\ln(p)}{T}}
  \end{align*}
  Rearranging the upper bound gives the first result. For the second
  result we simply take the limit and note that
  \begin{equation*}
    \lim_{T \rightarrow \infty} 2\sqrt{2}D\lg\left(\frac{2T}{T_0}\right)\sqrt{\frac{\ln(p)}{T}} %
    = 0.
  \end{equation*}
\end{proof}

\section{A Numerical Demonstration\label{sec:sim}}

We now demonstrate Algorithm~\ref{algo:infinite} on a small example of
$n = 2$ queues. In each time slot, the number of arriving customers is
geometrically distributed on $\Zbb_+$. For queue 1 the mean number of
arriving customers is 1 and for queue 2 the mean number of arriving
customers is 2. The arrivals are independent across time slots as well
as across queues. We take
\begin{equation*}
  b = \mat{cc}{0.5  & 0.3\\ {} & 0.2}
\end{equation*}
and
\begin{equation*}
  \Scal = \set{\mat{cc}{0 & 0}', \mat{cc}{1 & 0}', \mat{cc}{2 & 1}', \mat{cc}{0 & 2}'}.
\end{equation*}
This choice of $b$ shows that the expert scheduler prioritizes queue 1
over queue 2 and the expert also has a preference to not serve both
queues simultaneously. We simulate the system and run
Algorithm~\ref{algo:infinite} for $T = 10^6$ time slots with
$x_0~=~\mat{cc}{0&0}'$. The results are shown in Figure~\ref{fig:sim}.

\begin{figure}
  \begin{subfigure}{\columnwidth}
    \includegraphics[width=\columnwidth]{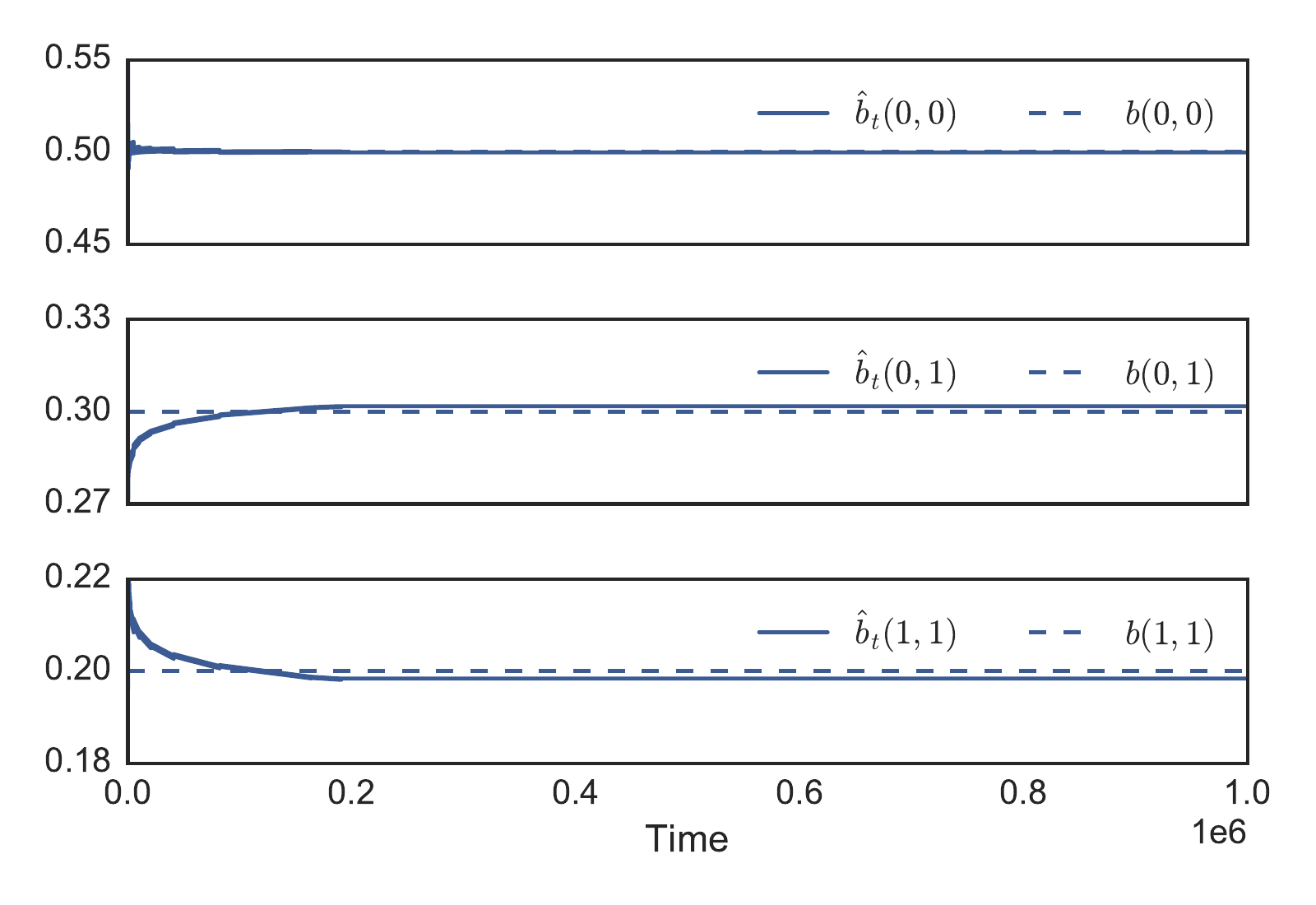}
    \caption{Evolution of $\hat b_t$\label{fig:convergence}}
  \end{subfigure}

  \begin{subfigure}{\columnwidth}
    \includegraphics[width=\columnwidth]{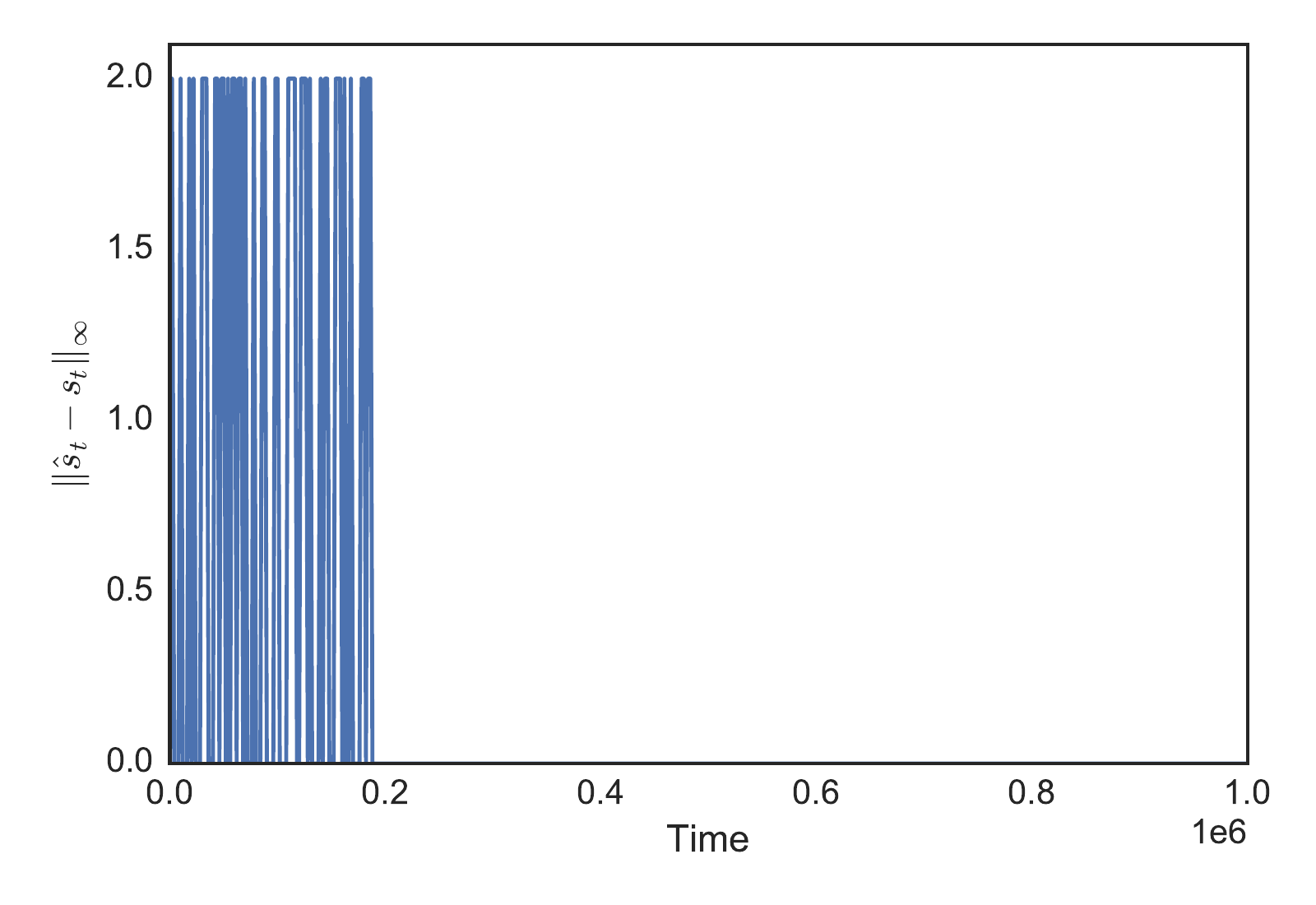}
    \caption{Error in learned scheduling decisions\label{fig:error}}
  \end{subfigure}

  \begin{subfigure}{\columnwidth}
    \includegraphics[width=\columnwidth]{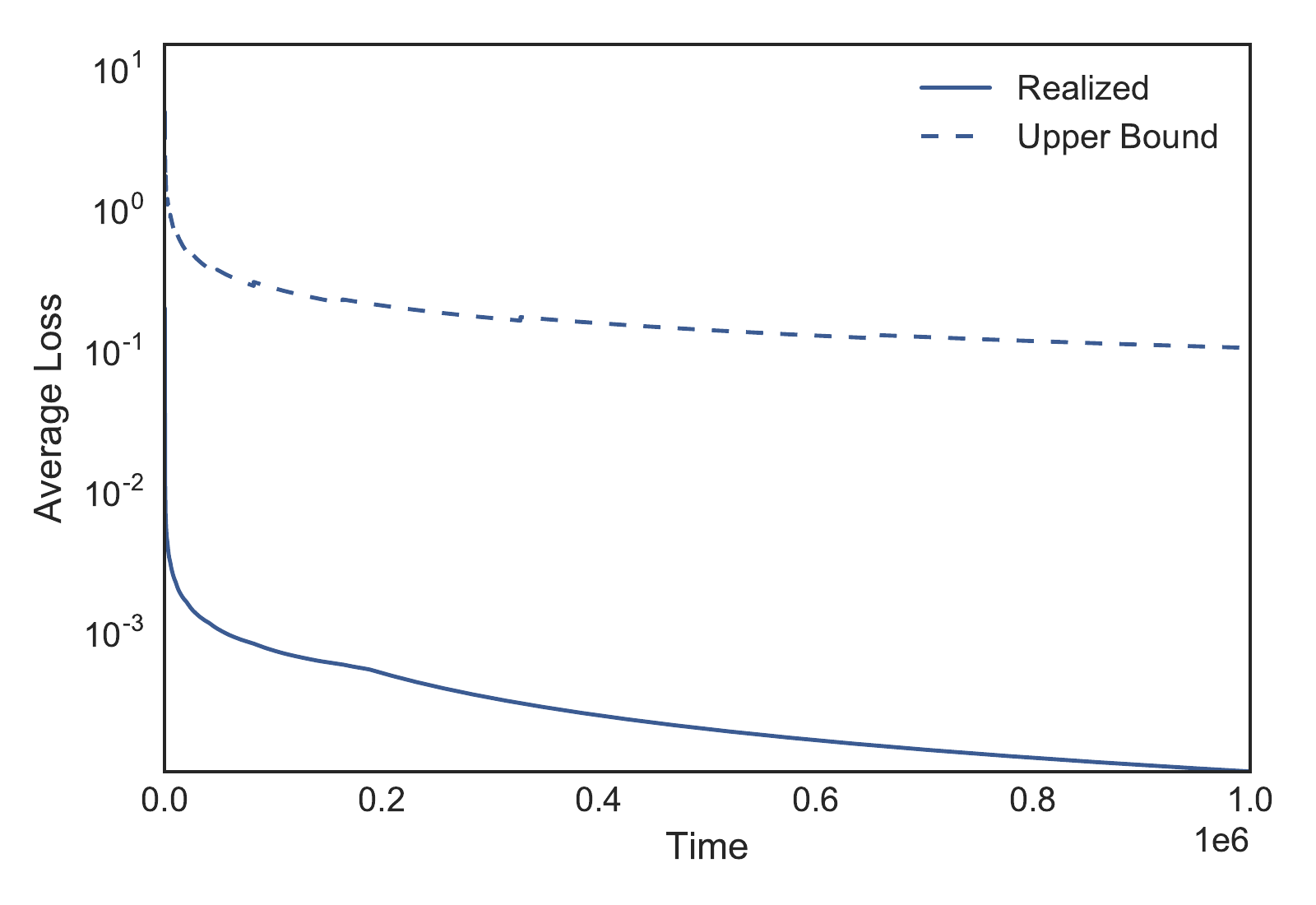}
    \caption{Average realized loss\label{fig:loss}}
  \end{subfigure}

  \caption{Output of Algorithm~\ref{algo:infinite} for the example in
    Section~\ref{sec:sim}\label{fig:sim}}
\end{figure}

First note that Figure~\ref{fig:convergence} shows that the $\hat b_t$
does not converge to $b$. We see that (to 4 decimal places)
\begin{equation*}
  \hat b_T = \mat{cc}{0.4998  & 0.3018 \\ {} & 0.1984}
\end{equation*}
and for the majority of the simulation these parameter estimates do
not change. The reason is that (as shown in in Figure~\ref{fig:error})
$\hat b_T$ yields the same scheduling decisions as $b$. The algorithm
learns to emulate the expert scheduler so the loss becomes zero and
the weights stop updating. This possibility was discussed at the
beginning of Section~\ref{sec:algos}.

Figure~\ref{fig:loss} which shows that while the average loss does
indeed tend to zero, the upper bound proved in
Theorem~\ref{thrm:infinite} is quite loose in this situation. This is
expected due to the generality of the theorem. This also means that
the concentration bound in Theorem~\ref{thrm:concentration} is quite
conservative. Indeed, for this simulation we see that $f_t(\epsilon) =
0$ for all $t$ and for any $\epsilon > 0$. In other words, no observed
loss ever exceeds the average case bound.

\section{Extensions\label{sec:extensions}}

We now discuss some extensions to our algorithms. We first note that
we could replace line 14 in Algorithm~\ref{algo:finite} with
\begin{equation*}
  w_{t+1} \leftarrow w_t \cdot \exp(-\eta w_t).
\end{equation*}
The new algorithm would be a Hedge-style algorithm and we would be
able to use apply other results
(e.g. \cite[Theorem~2.4]{Arora_MWU_2012}) to obtain similar upper
bounds on the average loss.

We also note that we could modify our algorithms and obtain tighter
upper bounds if we impose additional assumptions on the expert. For
example, the expert may have a fairly simple objective that leads to
prioritization of some queues over others. In this case, we would have
$b(i, j) = 0$ for $i \ne j$. Rather than having a triangular array of
$p$ parameter estimates, we could instead keep track of just $n$
estimates.  Since $\Theta(\ln(p)) = \Theta(\ln(n))$, the convergence
rate would not change but we would have smaller constant
factors. Other sparsity patterns could be handled in a similar
fashion. The diagonal case is slightly simpler because there would be
no need use $\sigma(i, j)$ to keep track of the appropriate signs.

As noted in Section~\ref{sec:dynamics}, a continuous-time PCS model
with heterogenous and stochastic service times was considered in
\cite{Armony_Cone_2003}. Our algorithms could be applied in this
setting as well by updating the algorithm immediately after customer
arrivals and departures rather than in discrete time slots. In
\cite{Armony_Cone_2003}, $B$ is diagonal and so we could apply the
simplifications mentioned above. Our theorems would still hold because
they not require that the state update happen at regularly spaced
intervals -- the algorithms merely require a stream of observed
backlogs and observed scheduling actions. 

\section{Conclusions and Future Work\label{sec:conclusions}}
In this paper we have proposed an algorithm that learns a scheduling
policy that emulates the behavior of an expert projective cone
scheduler. This offers a data-driven way of designing automated
scheduling policies that achieve the same goals as a human manager. We
have provided several theoretical guarantees and have numerically
demonstrated the efficacy of the algorithm on a simple example.

This paper opens the door for a few area of future work. One idea is
to provide tighter bounds that depend on the statistical properties of
the arrival process. A benefit of the current approach is that it does
not require any assumptions on the arrival process but the clear
downside is that the resulting bounds are quite loose. An algorithm
that uses information about the arrival process could have faster
convergence rates and tighter bounds. 

Another idea is to investigate the impact of an approximate
computation of $\mu$. As mentioned in Section~\ref{sec:algos}, in
large-scale problems, exactly computing $\mu(y; b)$ is generally a
difficult problem and heuristic approaches are typically taken in
practice. An area of future work would be to consider how such
approximation ``noise'' affects our ability to emulate the expert
scheduler.

\section*{References}
\bibliographystyle{elsarticle-num}
\bibliography{NMaster_LearningPCS}

\end{document}